\documentclass[twoside,11pt]{article}

\usepackage{jmlr2e}
\usepackage{times,amsmath,amssymb,graphicx,url,algorithm2e}
\usepackage[utf8]{inputenc} 
\usepackage[T1]{fontenc}    
\usepackage{hyperref}       
\usepackage{url}            
\usepackage{booktabs}       
\usepackage{amsfonts}       
\usepackage{nicefrac}       
\usepackage{microtype}      

\newcommand{\net}{\mathrm{net}}

\newcommand\Expect{{\mathbf E}}
\newcommand\E{{\mathbf E}}
\newcommand*{\eqdef}{\stackrel{\textup{def}}{=}}

\renewcommand\Re{{\mathbb R}}

\newcommand\calX{{\mathcal{X}}}

\newcommand\Id{{\mathrm{Id}}}
\newcommand\sigmamin{{\sigma_{\mathrm{min}}}}
\newcommand\sigmamax{{\sigma_{\mathrm{max}}}}
\newcommand\ev{{\mathrm{ev}}}

\jmlrheading{1}{2017}{}{}{}{bartlett17a}{Peter L. Bartlett, Steven N. Evans and Philip M. Long}

\ShortHeadings{Representing smooth functions as compositions of near-identity functions}{Bartlett, Evans and Long}
\firstpageno{1}

\begin{document}

\title{Representing smooth functions \\
as compositions of near-identity functions \\
with implications for deep network optimization \\}
%

\author{\name Peter L. Bartlett \email peter@berkeley.edu 
       \AND
       \name Steven N. Evans \email evans@stat.berkeley.edu 
       \AND
       \name Philip M. Long \email plong@google.com
       }

\editor{}

\maketitle

\begin{abstract}
We show that any smooth bi-Lipschitz $h$ can be represented exactly
as a composition $h_m \circ ... \circ h_1$ of 
functions $h_1,...,h_m$ that are close to the identity in
the sense that each $\left(h_i-\Id\right)$ is Lipschitz, and the
Lipschitz constant decreases inversely with the number $m$ of functions
composed. 
This implies that $h$ can be represented to any accuracy by
a deep residual network whose nonlinear layers compute functions with
a small Lipschitz constant.  Next, we consider nonlinear regression
with a composition of near-identity nonlinear maps.
We show that, regarding Fr\'echet derivatives with respect
 to the $h_1,...,h_m$,
any critical point of a quadratic criterion in this 
near-identity region must be a global minimizer.
In contrast, if we consider derivatives with respect to parameters of
a fixed-size residual network with sigmoid activation functions,
we show that there are near-identity critical points that
are suboptimal, even in the realizable case.
Informally, this means that functional gradient methods for
residual networks cannot get stuck at suboptimal critical points
corresponding to near-identity layers,
whereas parametric gradient methods for sigmoidal residual networks
suffer from suboptimal critical points in the near-identity region.
\end{abstract}


\section{Introduction}


The winner of the ILSVRC 2015 classification competition
used 
a new architecture called 
residual networks \citep{he2016deep}, which enabled very fast
training of very deep networks.  
These have since been widely adopted.
(As of this writing, the paper that introduced this
technique, published in 2015, has over 3700 citations.)
Deep networks express
models as the composition of transformations; residual networks depart 
from traditional deep learning models by using parameters to describe how
each transformation differs from the identity, rather than
how it differs from zero.

Motivated by this methodological advance,
\citet{HM17} recently considered
compositions of many linear maps,
each close to the identity map.
They showed that any matrix with spectral norm and
condition number bounded by constants can be represented as a product
of matrices $I+A_i$, where each $A_i$ has spectral norm $O(1/d)$.
They considered this non-convex parameterization for a linear
regression problem with additive Gaussian noise,
and showed that any critical point of the
quadratic loss for which the $A_i$ have sufficiently small spectral norm must
correspond to the linear transformation that generated the data.  
This raised the possibility that gradient descent with each layer initialized
to the identity might provably converge for this non-convex
optimization problem;  \citet{bartlett2018gradient} investigated
this, identifying sets of problems where this method
converges, and where it does not.

In this paper, 
we continue this line of research.
First, we identify a non-linear counterpart of Hardt and
Ma's results motivated by deep residual networks:
any smooth
%
bi-Lipschitz $h$ (that is, invertible Lipschitz map with
differentiable inverse) can be represented exactly as a composition of
functions $h_i$ that are close to the identity in the sense that each
$\left(h_i-\Id\right)$ is Lipschitz, and the Lipschitz constant
decreases inversely with the number of functions composed.  Since a
two-layer neural network with standard activation functions can
approximate arbitrary continuous functions, we can represent each
$h_i$ in the composition as $h_i=\Id+N_i$, where $N_i$ is computed by
a two-layer network in this way. The fact that $\left(h_i - \Id\right)$
has a small Lipschitz constant for deep networks shows that $N_i$ is
small, in the sense that it only needs to approximate a slowly changing
function.

The requirement in our analysis that $h$ is bi-Lipschitz generalizes
the assumption in the linear case studied by Hardt and Ma that the map
to be learned has a bounded condition number.  The practical strength,
and therefore relevance, of invertible feature maps in the non-linear
case is supported by success the reversible networks
\citep{maclaurin2015gradient,GRUG17,dinh2017density}.

For our second result, we consider a nonlinear regression problem
using a composition of near-identity nonlinear maps. If we consider
Fr\'echet derivatives with respect to the functions $h_i$ in the
composition, we show that any critical point of the quadratic
criterion must be the global optimum. In contrast, if each $h_i$ is a
two-layer net of the form $A_i \tanh(B_i x) + x$, analogously to the
architecture of \citep{he2016deep}, and we consider derivatives with
respect to the real-valued parameters, there are regression problems
that give rise to suboptimal critical points.  We discuss the
implications of this analysis in Section~\ref{s:discussion}.


A number of authors have investigated how using deep
architecture affects the set of functions computed by a network
\citep[see][]{montufar2014number,telgarsky2015representation,poole2016exponential,mhaskar2016learning}.
Our main results abstract away the parameterization and focus on the
expressiveness and of compositions of near-identity functions, along with
properties of their error landscapes.  

\section{Notation and Definitions}

Let $\Id$ denote the identity map on $\Re^d$, $\Id(x)=x$.
Throughout, $\|\cdot\|$ denotes a norm on $\Re^d$.
We also use $\|\cdot\|$ to denote an induced norm:
for a function $f:U\to V$, where $U$ and $V$ are normed
spaces with norms $\|\cdot\|_U$ and $\|\cdot\|_V$, we write
$
    \|f\| := \sup\left\{\frac{\|f(x)\|_V}{\|x\|_U}
    :x\in U,\,\|x\|_U>0\right\}.
$
Define the Lipschitz seminorm of $f$ as
  \[
    \|f\|_L := \sup\left\{\frac{\|f(x)-f(y)\|_V}{\|x-y\|_U}
      :x,y\in U,\,x\not= y\right\}.
  \]
Define the ball of radius $\alpha$ in a normed space
$(\calX,\|\cdot\|)$ as
$
    B_\alpha(\calX) = \left\{x\in\calX: \|x\|\le \alpha \right\}.
$
For a function $f:\Re^d\to\Re^d$, $Df$ denotes the Jacobian matrix,
that is, the matrix with entries
$J_{i,j}(x)=\partial f_i/\partial x_j(x)$.

For a functional $F:U\to V$ defined on Banach spaces $U$ and $V$,
recall that the Fr\'echet derivative of $F$ at $f\in U$ is the 
linear operator $DF(f):U\to V$ satisfying
$
    DF(f)(\Delta) = F(f+\Delta)-F(f) + o(\Delta),
$
that is,
  \[
    \lim_{\Delta\to 0} \frac{\left\|F(f+\Delta)-F(f)
        -DF(f)(\Delta)\right\|_V}
        {\|\Delta\|_U} = 0.
  \]
We use $D_fG(f,g)$ to denote the Fr\'echet derivative of $f\mapsto
G(f,g)$.

\section{Representation}


\begin{theorem}\label{theorem:main1}
For $R>0$, denote $\calX=B_R(\Re^d)$.
Let $h:\Re^d\to\Re^d$ be a differentiable, invertible map satisfying
the following properties:
  (a) Smoothness: for some $\alpha>0$
    and all $x,y,u\in\calX$,
     \begin{equation}
      \label{e:smoothness}
        \left\|\left(D h(y) - D h(x)\right)u\right\|
           \le \alpha\|y-x\|\|u\|;
     \end{equation}
    (b) Lipschitz inverse:
     for some $M>0$, $ \|h^{-1}\|_L\le M$;
    (c) Positive orientation: For some $x_0\in\calX$,
    $\det(Dh(x_0))>0$.

Then for all $m$, there are $m$ functions
$h_1,\ldots,h_m:\Re^d\to\Re^d$ satisfying,
for all $x\in\calX$,
$
    h_m\circ h_{m-1}\circ\cdots\circ h_{1}(x) = h(x)
$
and, on $h_{i-1}\circ\cdots \circ h_1(\calX)$,
$\left\|h_i-\Id\right\|_L=O(\log m/m)$.
\end{theorem}

Think of the functions $h_i$ as near-identity maps that might be
computed as 
 $
    h_i(x) = x+A\sigma(Bx)+b,
  $
where $A\in\Re^{d\times k}$ and $B\in\Re^{k\times d}$ are matrices,
$\sigma$ is a nonconstant nonlinearity, such as a sigmoidal function or
piecewise-linear function, applied component-wise, and $b\in\Re^d$ is
a vector. Although the proof constructs the $h_i$ as differentiable
(and even smooth) maps, each could, for example, be approximated to
arbitrary accuracy on the compact $\calX$ using a single layer of
ReLUs (for which $\sigma(\alpha)=\max\{0,\alpha\}$). (See, for
example, Theorem~1 in~\citep{h-acmfn-91}, and the comments in Section~3
of that paper about immediate generalizations to unbounded
nonlinearities.) In that case, the conclusion of the theorem implies
that $x\mapsto A\sigma(Bx)$ can be $O(\log m/m)$-Lipschitz.

Notice that the conclusion of the theorem does not require
the function $\left(h_i-\Id\right)$ to be small; shifting $h_i$ by
an arbitrary constant does not affect the Lipschitz property.

The constants hidden in the big-oh notation in the theorem
are polynomial in $1/\alpha$, $R$, $M$, $|\log\sigmamax(Dh(x_0))|$,
and $|\log\sigmamin(Dh(x_0))|$. (Here, $\sigmamin$ and $\sigmamax$
denote the smallest and largest singular values.)

The condition $\det(D(h(x_0)))>0$ is an unavoidable topological
constraint that arises because of the orientation of the identity map.
As \citet{HM17} argue in the linear context, if we view $h$ as a
mapping from raw representations to meaningful features, we can easily
set the orientation of $h$ appropriately (that is, so that
$\det(D(h(x_0)))>0$) without compromising the mapping's usefulness.

To prove Theorem~\ref{theorem:main1}, we prove the following
special case.

\begin{theorem}\label{theorem:deepsplit}
Consider an $h$ that satisfies the conditions of
Theorem~\ref{theorem:main1}
and also $h(0)=0$ and $D h(0)=I$.
Then for all $m$, there are $m$ functions
$h_1,\ldots,h_m:\Re^d\to\Re^d$ satisfying,
for all $x\in\calX$,
  \[
    h_m\circ h_{m-1}\circ\cdots\circ h_{1}(x) = h(x)
  \]
and, on $h_{i-1}\circ\cdots \circ h_1(\calX)$,
$\left\|h_i-\Id\right\|_L\le\epsilon$, provided
   $
    \epsilon \ge \frac{B\ln 2m}{m-1},
   $
where the constant $B$ depends on $\alpha$, $R$ and $M$.
\end{theorem}

To see that Theorem~\ref{theorem:main1} is a corollary, notice that we
can write
  $
    h(x) = Dh(x_0)\tilde h(x-x_0)+h(x_0),
  $
where
  $
    \tilde h(x)
      := \left(Dh(x_0)\right)^{-1}\left(h(x+x_0)-h(x_0)\right).
  $
Since $\tilde h$ satisfies $\tilde h(0)=0$ and $D\tilde h(0)=I$,
Theorem~\ref{theorem:deepsplit} shows that it can be expressed as a
composition of near-identity maps. Furthermore, the translations
$t_{x_0}(x):=x-x_0$ and $t_{h(x_0)}(x):=x-h(x_0)$ have
the property that $\left(t_{x_0}-\Id\right)$ is $0$-Lipschitz.
Finally, Theorem~2.1 in~\citet{HM17} shows that we can decompose
the Jacobian matrix
$ Dh(x_0) = (I+A_1)\cdots(I+A_m)$
with $\|A_i\|=O(\gamma/m)$ for
$ \gamma = |\log\sigmamax(Dh(x_0))|+|\log\sigmamin(Dh(x_0))|$,
and this implies that the linear map $A_i$ is
$O(\gamma/m)$-Lipschitz (see Lemma~\ref{lemma:Lipschitz},
Part~\ref{lemma:part:deriv} below).

Before proving Theorem~\ref{theorem:deepsplit}, we observe
that the smoothness property implies a bound on the accuracy
of a linear approximation, and a Lipschitz bound.  
The proof is in Appendix~\ref{a:smoothnessimplications}.
\begin{lemma}\label{lemma:smoothnessimplications}
    For $h$ satisfying 
    %
    the conditions of Theorem~\ref{theorem:deepsplit}
    and any $x,y\in\calX$,
      \begin{align*}
        \left\|h(y) - \left(h(x)+D h(x)(y-x)\right)\right\|
          & \le \frac{\alpha}{2}\|y-x\|^2,
      \end{align*}
  and $ \left\|h\right\|_L \le 1+\alpha R$.
\end{lemma}

\begin{proof} {\bf (of Theorem~\protect\ref{theorem:deepsplit})}
We give an explicit construction of the $h_1,\ldots,h_m$.
For $i=1,\ldots,m$, define $g_i:\calX\to\Re^d$ by
  $ g_i(x) = h(a_ix)/a_i$, 
where the constants $0<a_1<\cdots<a_m=1$ will be chosen later.
The $g_i$'s can be viewed as functions that interpolate between the
identity (which is $Dh(0)$, the limit as $a$ approaches zero of
$h(ax)/a$) and $h$ (which is $g_m$, because $a_m=1$).
Note that $g_i$ is invertible on $\calX$, with
  $ g_i^{-1}(y) = h^{-1}(a_iy)/a_i $
for $y\in g_i(\calX)$. Define $h_1=g_1$ and, for $1<i\le m$,
define $h_i:g_{i-1}(\calX)\to\Re^d$ by
  $
    h_i(x) = g_i(g_{i-1}^{-1}(x)),
  $
so that $h_i\circ h_{i-1}\circ\cdots\circ h_{1} = g_i$ and
in particular $h_m\circ\cdots\circ h_{1} = g_m = h$.
It remains to show that, for a suitable choice of
$a_1,\ldots,a_{m-1}$, the $h_i$ satisfy the Lipschitz condition.

%
We have
  \begin{align*}
    \lefteqn{\|h_1(x)-x-(h_1(y)-y)\|} & \\
      & = \frac{1}{a_1}\left\|h(a_1x)-a_1x-(h(a_1y)-a_1y)\right\| \\
      & = \frac{1}{a_1}\left\|D h(a_1y)(a_1x-a_1y) -(a_1x-a_1y)
        +\left(h(a_1x)-\left(h(a_1y)+D h(a_1y)(a_1x-a_1y)\right)\right)
        \right\| \\
      & = \frac{1}{a_1}\left\|\left(D h(a_1y) -Dh(0)\right)(a_1x-a_1y)
        +\left(h(a_1x)-\left(h(a_1y)+D h(a_1y)(a_1x-a_1y)\right)\right)
        \right\| \\
      & \le a_1\alpha \left(\|y\|\|x-y\|+\frac{1}{2}\|x-y\|^2\right) 
        \;\;\; 
   \mbox{(by (\ref{e:smoothness}) and Lemma~\ref{lemma:smoothnessimplications})}
     \\
    & 
  \le 2R a_1\alpha \|x-y\|.
  \end{align*}

Now, fix $i>1$ and $u,v\in\calX$ and set
$x=g_{i-1}(u)$ and $y=g_{i-1}(v)$.
Then we have
  \begin{align}
    \left\|h_i(y)-y-(h_i(x)-x)\right\|
      & = \left\|g_i(v)-g_{i-1}(v)-(g_i(u)-g_{i-1}(u))\right\| \notag\\
      & = \frac{1}{a_i}\left\|
        h(a_iv)-\frac{a_i}{a_{i-1}}h(a_{i-1}v)
        -\left(h(a_iu)-\frac{a_i}{a_{i-1}}h(a_{i-1}u)\right)\right\|.
    \label{equation:hiLip}
  \end{align}
We consider two cases: when $y$ and $x$ are close, and when they are
distant. First, suppose that $\|y-x\|\le a_i-a_{i-1}$.
  \begin{align*}
    \lefteqn{\left\|h_i(y)-y-(h_i(x)-x)\right\|} & \\
      & = \frac{1}{a_i}\left\|
        h(a_iv) - h(a_iu)
        -\frac{a_i}{a_{i-1}}\left(h(a_{i-1}v)
        -h(a_{i-1}u)\right)
        \right\| \\
      & = \frac{1}{a_i}\left\| \rule{0pt}{17pt} a_iD h(a_iu)(v-u)
        + h(a_iv) - \left(h(a_iu) + a_iD h(a_iu)(v-u)\right)
        \right. \\*
        & \qquad \left. {}
        \hspace{0.1in}
        -\frac{a_i}{a_{i-1}}\left(
        \rule{0pt}{17pt}
        a_{i-1}D h(a_{i-1}u)(v-u)
        + h(a_{i-1}v) - \left(h(a_{i-1}u) +a_{i-1}D h(a_{i-1}u)(v-u)\right)
        \right) \right\| \\
      & \le \left\|\left(D h(a_iu)-D
        h(a_{i-1}u)\right)(v-u)\right\|
        + \frac{\alpha}{2}(a_i+ a_{i-1})\|v-u\|^2 
           \;\;\;\;\; 
      \mbox{(applying Lemma~\ref{lemma:smoothnessimplications} twice)}
                \\
      & \le \alpha(a_i-a_{i-1})\|u\|\|v-u\|
        + \frac{\alpha}{2}(a_i+ a_{i-1})\|v-u\|^2 
           \;\;\;\;\; 
   \mbox{(by (\ref{e:smoothness}))}
               \\
      & \le \alpha\left(R(a_i-a_{i-1})
        + \frac{1}{2}(a_i+ a_{i-1})\|v-u\|\right)\|v-u\|.
  \end{align*}
Also, we can relate $\|v-u\|$ to $\|y-x\|$ via the Lipschitz property
of $h^{-1}$:
  \begin{align*}
    a_{i-1}\|v-u\|
      & = \| h^{-1}(h(a_{i-1}v))
        -h^{-1}(h(a_{i-1}u)) \| 
      \le M\|h(a_{i-1}v)-h(a_{i-1}u)\|,
  \end{align*}
so
  \begin{align}
    \|y-x\|
      & = \frac{1}{a_{i-1}}\left\|h(a_{i-1}v)-h(a_{i-1}u)\right\| 
      \ge \frac{1}{M}\|v-u\|.
    \label{equation:y-xvsv-u}
  \end{align}
Combining,
and using the assumption $\|y-x\|\le a_i-a_{i-1}$,
  \begin{align}
   \nonumber
    \left\|h_i(y)-y-(h_i(x)-x)\right\|
      & \le \alpha M\left(R(a_i-a_{i-1})
        + \frac{1}{2}(a_i+ a_{i-1}) M\|y-x\|\right)\|y-x\| \\
    \label{e:close}
      & \le (a_i-a_{i-1})\alpha M\left(R + M\right)\|y-x\|.
  \end{align}

Now suppose that $\|y-x\|>a_i-a_{i-1}$. From~\eqref{equation:hiLip},
we have
  \begin{align}
   \nonumber
     \lefteqn{\left\|h_i(y)-y-(h_i(x)-x)\right\|} & \\
   \nonumber
      & = \frac{1}{a_i}\left\|
        h(a_iv)-\frac{a_i}{a_{i-1}}h(a_{i-1}v)
        -\left(h(a_iu)-\frac{a_i}{a_{i-1}}h(a_{i-1}u)\right)\right\| \\
   \nonumber
      & = \frac{1}{a_i}\left\|
        \frac{a_{i-1}-a_i}{a_{i-1}} (h(a_{i-1}v) - h(a_{i-1}u))
        + h(a_iv)-h(a_{i-1}v)-
        \left(h(a_iu)-h(a_{i-1}u)\right)\right\| \\
   \nonumber
      & = \frac{1}{a_i}\left\|
        \frac{a_{i-1}-a_i}{a_{i-1}} (h(a_{i-1}v) - h(a_{i-1}u))
        \right. \\*
   \nonumber
      & \qquad {}
       \hspace{0.1in}
        + h(a_iv)-\left(h(a_{i-1}v)+
        D h(a_{i-1}v)(a_iv-a_{i-1}v)\right) \\*
   \nonumber
      & \qquad {}
       \hspace{0.1in}
        {} - \left(h(a_iu)-\left(h(a_{i-1}u)
        +D h(a_{i-1}u)(a_iu-a_{i-1}u)\right)
        \right) \\*
   \nonumber
      & \qquad {}
       \hspace{0.1in}
        \left.{}  -
        D h(a_{i-1}v)(a_iv-a_{i-1}v)
        +D h(a_{i-1}u)(a_iu-a_{i-1}u)\right\| \\*
   \nonumber
      & \le \frac{a_i-a_{i-1}}{a_i}L\|v-u\| 
        + \frac{1}{a_i}\left\|
        D h(a_{i-1}v)(a_iv-a_{i-1}v)
        -D h(a_{i-1}u)(a_iu-a_{i-1}u)\right\| \\*
   \label{e:far}
      & \qquad {}
        +\frac{1}{a_i}\frac{\alpha}{2}(a_i-a_{i-1})^2
        \left(\|v\|^2+\|u\|^2\right)
  \end{align}
where, in the first term, we have used the Lipschitz property from
Lemma~\ref{lemma:smoothnessimplications},
with $L=(1+\alpha R)$. But
  \begin{align*}
    \lefteqn{\frac{1}{a_i}\left\|
        D h(a_{i-1}v)(a_iv-a_{i-1}v)
        -D h(a_{i-1}u)(a_iu-a_{i-1}u)\right\|} & \\
    & = \frac{a_i-a_{i-1}}{a_i}\left\|
        D h(a_{i-1}v)(v) -D h(a_{i-1}u)(u)\right\| \\
    & = \frac{a_i-a_{i-1}}{a_i}\left\|
        v-u + \left(D h(a_{i-1}u)-D h(0)\right)(v-u)
        +\left(D h(a_{i-1}v)-D h(a_{i-1}u)\right)v\right\| \\
    & \le \frac{a_i-a_{i-1}}{a_i}\left(1+\alpha a_{i-1}\left(\|u\| + 
        \|v\|\right)\right)\|v-u\|,
  \end{align*}
by (\ref{e:smoothness}).  Substituting into (\ref{e:far}), and
using~\eqref{equation:y-xvsv-u} together with the assumption that
$\|y-x\|>a_i-a_{i-1}$,
  \begin{align*}
    \lefteqn{\left\|h_i(y)-y-(h_i(x)-x)\right\|} & \\
      & \le \frac{a_i-a_{i-1}}{a_i}\left(LM +
    M\left(1+a_{i-1}\alpha (\|u\| + \|v\|)\right)
        +\frac{\alpha}{2}\left(\|v\|^2+\|u\|^2\right)\right)
        \|y-x\| \\
      & \le \frac{a_i-a_{i-1}}{a_i}\left(M(L+1+2R\alpha)
        +\alpha R^2\right) \|y-x\|.
  \end{align*}
Combining with (\ref{e:close}),
it suffices to choose $a_1$ to satisfy
$
    a_1\le \frac{\epsilon}{2\alpha R},
$
and $a_2,\ldots,a_{m-1}$ to satisfy, for $i>1$,
$
    \frac{a_i-a_{i-1}}{a_i} \le \frac{\epsilon}{B},
$
where
  \begin{align*}
    B & = \max\left\{ \alpha M(R + M),
    M(L+1+2R\alpha)+\alpha R^2\right\}.
  \end{align*}
If we choose $0<c<1$ and set $a_i=(1-c)^{m-i}$ for $i=1,\ldots,m$,
then these conditions are equivalent to
$c \le \epsilon/B$ and
  \begin{align*}
    &(1-c)^{m-1} \le \frac{\epsilon}{2\alpha R} 
  \qquad  \Leftrightarrow \qquad
     1-\left(\frac{\epsilon}{2\alpha R}\right)^{1/(m-1)} \le c.
  \end{align*}
Thus, it suffices if
  \begin{align*}
&
    \frac{\epsilon}{B}
    \ge
    1-\left(\frac{\epsilon}{2\alpha R}\right)^{1/(m-1)}   
   \qquad \Leftarrow \qquad
    \epsilon \ge \frac{B}{m-1}\ln
    \frac{2\alpha R}{\epsilon} \\
&
   \Leftarrow \qquad
    \epsilon \ge \frac{B}{m-1}\max\left\{1,
    \ln \frac{2\alpha R m}{B}\right\} 
\qquad  \Leftarrow \qquad
    \epsilon \ge \frac{B\ln 2m}{m-1},
  \end{align*}
using the inequality $1-x\le\ln (1/x)$,
which follows from convexity of $\ln(1/x)$.
\end{proof}

\section{Zero Fr\'echet derivatives with deep compositions}

The following theorem is the main result of this section. It shows
that if
a composition of near-identity maps has zero Fr\'echet derivatives
of a quadratic criterion
with respect to the functions in the composition, then the composition
minimizes that criterion. That is, all critical points of this kind are
global minimizers; there are no saddle points or suboptimal local
minimizers in the near-identity region.

\begin{theorem}\label{theorem:optimization}
  Consider a distribution $P$ on $\calX\times\calX$, and
  define the criterion
\[
      Q(h) = \frac{1}{2}\Expect_{(X,Y)\sim P}\left\|h(X)-Y\right\|_2^2.
\]
  Define a conditional expectation $h^*(x)=\Expect[Y|X=x]$, so that
  $h^*$ minimizes $Q$.
  Consider the function computed by an $m$-layer network
  $h = h_m\circ \cdots \circ h_1$, and suppose that, for some
  $0<\epsilon<1$ and all $i$, $h_i$ is differentiable,
  $\|h_i\|<\infty$, and $\left\|h_i-\Id\right\|_L\le \epsilon$.
  Suppose that $\|h-h^*\|<\infty$. Then for all $i$,
    \[
      \inf_{\Delta\in B_1}
      D_{h_i}Q(h)(\Delta) \le \frac{- (1-\epsilon)^{m-1}}{\|h-h^*\|}
        \left(Q(h) - Q(h^*)\right).
    \]
  Thus, if $h$ is a critical point of $Q$,
  that is, for all $i$, $D_{h_i}Q(h) = 0$,
  we must have $Q(h) = Q(h^*)$.
\end{theorem}

The theorem defines the expected quadratic loss under an arbitrary
joint distribution, but in particular it could be a discrete distribution
that is uniform on a training set.

Notice that if $h^*$ satisfies the properties of
Theorem~\ref{theorem:main1}, then it can be represented as a
composition of $h_i$ with the required properties.
If it cannot, then the theorem shows that the near-identity
region will not contain critical points. The only
property we require of $h^*$ is the boundedness condition
$\|h-h^*\|<\infty$. From the definition of the induced norm,
this implicitly assumes that $h(0)=h^*(0)$ and
that $h^*$ is differentiable at $0$. In the context of learning
embeddings, it seems reasonable to fix the embedding's value at one
input vector, and express its value elsewhere relative to that
value.

Notice also that, although the theorem requires differentiability of
the $h_i$, it is only important for various derivatives to be defined.
In particular, a network with non-differentiable but Lipschitz
activation functions, like a ReLU network, could be approximated to
arbitrary accuracy by replacing the ReLU nonlinearity with a
differentiable one. The conclusions of the theorem apply to any
critical point at a differentiable approximation of the ReLU
network.

\begin{lemma}\label{lemma:Lipschitz}
  Suppose $\left\|f-\Id\right\|_L\le\alpha<1$.
  \begin{enumerate}
    \item\label{lemma:part:isometry}
    $(1-\alpha)\|x-y\|\le\|f(x)-f(y)\| \le (1+\alpha)\|x-y\|$.
    \item\label{lemma:part:inverse}
    $f$ is invertible and
    $\left\|f^{-1}-\Id\right\|_L\le\alpha/(1-\alpha)$.
    \item\label{lemma:part:deriv}
    For $F(g)=f\circ g$, $\|D F(g)-\Id\|\le\alpha$, and hence
    $\left\|D F(g)-\Id\right\|_L\le\alpha$.
  \end{enumerate}
\end{lemma}

\begin{proof}
Part~\ref{lemma:part:isometry}:
The triangle inequality and the Lipschitz property gives
  \begin{align*}
    \|x-y\|
        & \le \|f(x)-f(y)\| + \|f(x)-x - (f(y)-y)\| 
        \le \|f(x)-f(y)\| + \alpha\|x - y\|.
  \end{align*}
Similarly,
  \begin{align*}
    \|f(x)-f(y)\|
        & \le \|x-y\| + \|f(x)-x - (f(y)-y)\| 
        \le (1+\alpha)\|x - y\|.
  \end{align*}
Part~\ref{lemma:part:inverse}:
For $\alpha<1$, the inequality
$\|f(x)-f(y)\|\ge (1-\alpha)\|x-y\|$
of Part~\ref{lemma:part:isometry} shows that $f$ is invertible.
Together with the Lipschitz property, this also shows that
  \begin{align*}
    \|x-y-(f(x)-f(y))\|
        & \le \alpha \|x-y\| 
        \le \frac{\alpha}{1-\alpha} \|f(x)-f(y)\|,
  \end{align*}
which, since $(f^{-1}-\Id) (f(x)) = x - f(x)$, gives
$\left\|f^{-1}-\Id\right\|_L\le\alpha/(1-\alpha)$.\\
Part~\ref{lemma:part:deriv}:
From the definition of 
$D F(g)$,
  $
    \lim_{\Delta\to 0}
        \left\|F(g+\Delta)-F(g)-D F(g)(\Delta)\right\|/
        \|\Delta\| = 0.
  $
We can write, for any $\Delta$ with $\|\Delta\|<\infty$,
  \begin{align*}
    \left\| \Delta - D F(g)(\Delta)\right\| 
      & = \lefteqn{\left\| \Delta 
              + F(g+\Delta) - F(g+\Delta) + F(g) - F(g) + g - g
                    - D F(g)(\Delta)\right\|} \\
      & \le \left\|F(g+\Delta)-F(g)-D F(g)(\Delta)\right\| \\*
      & \qquad {}
        + \left\|f\circ(g+\Delta) - (g+\Delta) - (f\circ g-g) \right\| \\
      & = o\left(\|\Delta\|\right)
        + \sup_x \frac{\left\|f\circ(g+\Delta)(x) - (g+\Delta)(x)
          - (f\circ g-g)(x) \right\|}{\|x\|} \\
      & = o\left(\|\Delta\|\right)
        + \alpha\sup_x \frac{\left\|\Delta(x)\right\|}{\|x\|} 
      = o\left(\|\Delta\|\right)
        + \alpha\|\Delta\|.
  \end{align*}
Hence, $\|D F(g)-\Id\|\le\alpha$.
Since $\left(D F(g)-\Id\right)$ is a linear functional,
this also shows that it is $\alpha$-Lipschitz:
  \begin{align*}
    \left\|D F(g)(\Delta_1)-\Delta_1 - \left(D F(g)(\Delta_2) -
    \Delta_2\right)\right\| 
      = \left\|D F(g)(\Delta_1-\Delta_2) - (\Delta_1-\Delta_2)\right\|
     \le \alpha\|\Delta_1-\Delta_2\|.
  \end{align*}
\end{proof}

\begin{proof} {\bf (of Theorem~\ref{theorem:optimization})}
  From the projection theorem,
    \begin{align*}
      Q(h)
        & = \frac{1}{2}\Expect_{(X,Y)\sim P}\left\|h(X)-Y\right\|_2^2 \\
        & = \frac{1}{2}\Expect\left\|h(X)-h^*(X)\right\|_2^2 + 
      \frac{1}{2}\Expect_{(X,Y)\sim P}\left\|h^*(X)-Y\right\|_2^2.
    \end{align*}
  Fix $1\le i\le m$.  
  To analyze the effect of
  changing the function $h_i$ on $Q(h)$
  by applying the chain rule
  for Fr\'echet derivatives,
  we trace the effect of changing $h_i$ on $h$ by describing $h$ as the
  result of the composition of a sequence of functionals, which
  map functions to functions.
  In particular, we write
    $
      h = H_m\circ\cdots\circ H_{i+1}\circ G_i(h_i),
    $
  where $H_j(g) := h_j\circ g$ for $i<j\le m$,
  $h_i^j:=h_j\circ\cdots \circ h_i$ for $i\le j\le m$,
  and $G_i(g)=g\circ h_{i-1}\circ\cdots\circ h_1$.
  Now, using the chain rule for Fr\'echet derivatives,
  \begin{align*}
    D_{h_i}Q(h)
      & = \Expect\left[ (h(X)-h^*(X))\cdot \ev_X\circ D_{h_i}h\right] \\
      & = \Expect\left[ (h(X)-h^*(X))\cdot \ev_X\circ
      D H_m(h_i^{m-1})\circ\cdots\circ
      D H_{i+1}(h_i^i)\circ D G_i(h_i)\right],
  \end{align*}
  where $\ev_x$ is the evaluation functional, 
  $\ev_x (f):= f(x)$.
  From the definition of the Fr\'echet derivative, 
  $D G_i(g)$ always satisfies
  \begin{align}
\nonumber
   0 
   & =  \lim_{\Delta\to 0} \frac{\left\|G_i(g+\Delta)-G_i(g)
                                  -D G_i(g)(\Delta)\right\|}
                 {\|\Delta\|}  \\
\nonumber
   & =  \lim_{\Delta\to 0} \frac{\left\|
(g+\Delta)\circ h_{i-1}\circ\cdots\circ h_1
            -g\circ h_{i-1}\circ\cdots\circ h_1
                                  -D G_i(g)(\Delta)\right\|}
                 {\|\Delta\|}  \\
\label{e:limit}
   & =  \lim_{\Delta\to 0} \frac{\left\|
    \Delta \circ h_{i-1}\circ\cdots\circ h_1
                                  -D G_i(g)(\Delta)\right\|}
                 {\|\Delta\|}.
  \end{align}
The definition of the Fr\'echet derivative also
implies that $D G_i(g)$ is linear, as is the functional
$F$ defined by $F(\Delta) = \Delta \circ h_{i-1}\circ\cdots\circ h_1$.
If $F$ and $D G_i(g)$ were unequal, progressively scaling down an input
on which they differ would scale down the difference by the same amount,
contradicting~\eqref{e:limit}.  Thus $F = D G_i(g)$,
%
  which in turn implies
  \begin{align*}
    D_{h_i}Q(h)(\Delta)
      & = \Expect \left[(h(X)-h^*(X)) \cdot \ev_X\circ
      {} D H_m(h_i^{m-1})\circ\cdots\circ
      D H_{i+1}(h_i^i)\circ\Delta\circ h_{i-1}\circ\cdots\circ
      h_1 \right] \\
      & = \Expect \left[(h(X)-h^*(X)) 
      {} \cdot D H_m(h_i^{m-1})\circ\cdots\circ
      D H_{i+1}(h_i^i)\circ\Delta\circ h_{i-1}\circ\cdots\circ
      h_1(X)\right]. \\
  \end{align*}
  For all $j$, since $\left(h_j-\Id\right)$ is $\epsilon$-Lipschitz,
  Lemma~\ref{lemma:Lipschitz} implies $h_j$ is invertible.
  The lemma also implies that, for all $j$,
  $\left(DH_j(h_i^{j-1})-\Id\right)$ is $\epsilon$-Lipschitz,
  and hence that $DH_j(h_i^{j-1})$ is also invertible.
  Because these inverses exist, we can define
    \[
      \Delta = c \left(D H_m(h_i^{m-1})\circ\cdots\circ
      D H_{i+1}(h_i^i)\right)^{-1}\circ (h^*-h) \circ 
        \left(h_{i-1}\circ\cdots\circ h_1\right)^{-1},
    \]
  where we pick the scalar $c>0$ so that $\|\Delta\|=1$.
  This choice ensures that
  \begin{align}
      D H_m(h_i^{m-1})\circ\cdots\circ
      D H_{i+1}(h_i^i)\circ\Delta\circ h_{i-1}\circ\cdots\circ
      h_1 = c(h^*-h),
    \label{equation:Deltadef}
  \end{align}
  and hence
$
    D_{h_i}Q(h)(\Delta)
      = -c\Expect \|h(X)-h^*(X)\|_2^2.
$
  Since $\|\Delta\|=1$, for all $\gamma>0$ there is a $y$ with
  $\|\Delta(y)\|\ge(1-\gamma)\|y\|$. Define
 $
      x=(h_{i-1}\circ\cdots \circ h_1)^{-1}(y).
$
  Then, using the definition of the induced norm and
  Equation~\eqref{equation:Deltadef},
  we have
    \begin{align*}
      c\|h-h^*\|
         \ge c\frac{\|h(x)-h^*(x)\|}{\|x\|} 
      = \frac{1}{\|x\|}\left\|D H_m(h_i^{m-1})\circ\cdots\circ
        D H_{i+1}(h_i^i)\circ\Delta(y)\right\|.
    \end{align*}
  Recalling that all $\left(D H_j(h_i^j)-\Id\right)$ and
  $\left(h_j-\Id\right)$ are $\epsilon$-Lipschitz, we can apply
  Lemma~\ref{lemma:Lipschitz}:
    \begin{align*}
      c\|h-h^*\|
        & \ge (1-\epsilon)^{m-i}\frac{\|\Delta(y)\|}{\|x\|} \\
        & \ge (1-\epsilon)^{m-i}(1-\gamma)\frac{\|y\|}{\|x\|} \\
        & = (1-\epsilon)^{m-i}(1-\gamma)\frac{\|h_{i-1}\circ\cdots
          \circ h_1(x)\|}{\|x\|} \\
        & \ge (1-\epsilon)^{m-1}(1-\gamma).
    \end{align*}
  Taking the limit as $\gamma\to 0$ implies the result.
\end{proof}

\section{Bad critical points for sigmoid residual nets}
\label{s:stuck}

Theorem~\ref{theorem:optimization} may be paraphrased to say that
residual nets cannot have any bad critical points in the near-identity
region, when we consider Fr\'echet derivatives.  In this section, we
show that when we consider gradients with respect to the parameters of
a fixed-size residual network with sigmoid activation functions, the
corresponding statement is not true.

For a depth $m$, width $d$ and size $k$, the 
{\em $(m,d,k)$ $\tanh$ residual network} $N$ with parameters
$\theta=(A_1,..., A_m,B_1,...,B_m)$
computes the function $h_\theta \eqdef h_m \circ ... \circ h_1$, where each 
layer $h_i$ is defined by $h_i(x) = A_i \tanh (B_i x) + x$, with
$A_1,..., A_m\in\Re^{d\times k}$ and $B_1,...,B_m\in\Re^{k\times d}$,
and we define $\tanh$ of a vector as the component-wise application of
$\tanh$.

To gain an intuitive understanding of the existence of suboptimal
critical points, consider the following two properties of networks
with $\tanh$ nonlinearities.  First, there are finitely many simple
transformations (such as permutations of hidden units, or negation of
the input and output parameters of a unit) that leave the network
function unchanged. Second, apart from these transformations,
two networks with different parameter values compute different functions.
(This was shown for generic parameter values and $\tanh$ networks of
arbitrary depth by 
\citet{fefferman1994reconstructing}, and 
improved by 
\citet{albertini1992neural} for
the special case of two-layer networks.) Then for any globally optimal
parameter value, there is a simple transformation that is also
globally optimal. Consider a path between these two parameter values
that minimizes the maximum value of the criterion along the path. 
(It is not hard to construct a scenario in which such a minimax path
exists.)
The
maximizer must be a suboptimal critical point. The proof we give of the
following theorem is more direct, relying on specific properties of
the $\tanh$ parameterization, but we should expect a similar
result to apply to networks with other nonlinearities and
parameterizations, provided 
functions have multiple isolated distinct
representations as in the case of $\tanh$ networks.

The proof leverages the fact
that, while Theorem~\ref{theorem:optimization} rules out the
possibility of bad critical points arising from interactions between
the layers $h_1,...,h_m$, they may still arise due to the dynamics of
training an individual $h_i$.

\begin{theorem}
\label{t:stuck}
For any $\epsilon > 0$, any dimension $d$, width $k$ and a depth $m$, for
all $(m,d,k)$ tanh residual networks $N^*$ that do not compute the
identity function,
there is an $R > 0$, 
and a joint distribution $P$, 
over $B_R(\Re^d) \times B_R(\Re^d)$, such that
$N^*$ has parameter $\theta^*$ that minimizes
$Q(\theta) =
 \frac{1}{2}\Expect_{(X,Y)\sim P}\left\|h_{\theta}(X)-Y\right\|_2^2$,
and
the layers $h^*_i$ of $N^*$ satisfy
  $\|h^*_i-\Id\|_L\le\epsilon$,
and there is a (m,d,k) tanh residual network
$N$ 
such that
$N$ has parameter $\theta$ that is a critical point for $Q$,
$N$ has layers $h_i$ that satisfy $\|h_i - \Id\|_L\le\|h^*_i-\Id\|_L$,
but $Q(\theta) > Q(\theta^*)$.
\end{theorem}

\begin{proof}
Let $A_1,..., A_m, B_1,...,B_m$ be the parameters of $N^*$
and define $f_i$ by
$f_i(x) = A_i \tanh(B_i x) + x$, and $f_{N^*} = f_m \circ ... \circ f_1$.


Let $P$ be any joint distribution over examples
$(x,y)$ such that
\begin{itemize}
\item $\E(x) = 0$, 
\item $y = f_{N^*}(x)$ with probability $1$, and
\item $\E(|| y - x ||^2) > 0$.
\end{itemize}

Let $N$ be the network with all-zero parameters.  We claim
that $N$ a saddle point of $Q$.  Choose a weight $w$ in $N$,
between nodes $u$ and $v$.  Let $\net$
be presquashed linear combination of the nodes providing an input
to $v$, so that $v = \tanh(\net)$.  For a particular
$(x,y)$, if we define $Q_{(x,y)}(N) = (f_{N} (x) - y)^2$, then
\[
\frac{ \partial Q}{\partial w} 
  = \E_{(x,y) \sim P} \left(\frac{ \partial Q_{(x,y)}}{\partial w}\right).
\]
Furthermore,
\[
\frac{ \partial Q_{(x,y)}}{\partial w}
 = \frac{ \partial Q_{(x,y)}}{\partial v} 
   \frac{ \partial v}{\partial \net }
   \frac{ \partial \net}{\partial w}.
\]
If all of the weights are zero, however,
$\frac{ \partial Q_{(x,y)}}{\partial v}$ and 
$\frac{ \partial v}{\partial \net}$ 
do not depend on the input, so that 
$\frac{ \partial Q_{(x,y)}}{\partial w}$ is proportional to 
$\frac{ \partial \net}{\partial w} = u$.  However, again, since all
of the weights are zero, by induction, there is a component $x$ of
the input such that $u = x$, which implies $\E(u) = 0$.  Since
$\frac{ \partial Q_{(x,y)}}{\partial v}$ 
and $\frac{ \partial v}{\partial \net }$ are constant, and 
$\E\left(\frac{ \partial \net}{\partial w}\right) = 0$, we have
\[
\frac{ \partial Q}{\partial w} 
  = \E_{(x,y) \sim P} \left(\frac{ \partial Q_{(x,y)}}{\partial w}\right) = 0,
\]
and therefore $N$ is a critical point.

It remains to show that $Q(N) > Q(N^*)$.
Since all of the weights
of $N$ are $0$, $N$ computes the identity function.  Since
$\E(|| y - x ||^2) > 0$, this means $Q(N) > 0 = Q(N^*)$.  
\end{proof}

\section{Discussion}
\label{s:discussion}

Informally, Theorem~\ref{theorem:optimization} says that, if
near-identity behavior on each layer is maintained, for example
through regularization or early stopping, optimization with deep
residual nets using gradient descent cannot get stuck due to
interactions between the layers.

Theorem~\ref{theorem:optimization} also has consequences for
algorithms that optimize residual networks with a countably infinite
number of parameters on each layer, for instance by sequentially
adding units. Such algorithms have been studied for two-layer networks;
see \citep{bach2014breaking,bengio2006convex,lee1996efficient}.
Indeed, standard denseness results (see the remarks after Theorem 1) show
that any downhill direction in function space can be approximated by a
single layer network, such as a linear combination of sigmoid units,
or of ReLU units. (Notice that, because of the Lipschitz constraint, the
class of functions that such a network can compute has bounded
%
statistical complexity.) 
%
Thus, for residual networks for which the nonlinear components are
computed by standard architectures with a {\em variable} size, the main
result shows that every critical point (with this parameterization)
satisfying the Lipschitz constraint is a global optimum.


Typically, however, the size of the hidden layers is fixed.
As we have seen in Section~\ref{s:stuck}, while
Theorem~\ref{theorem:optimization} rules out the possibility of bad
critical points due to interaction between the layers,
there {\em can} be bad critical points due to dynamics within an
individual layer. On the other hand, 
a regularizer that promotes in each layer a small Lipschitz norm
deviation from the identity function
allows the use of a large number of hidden units while
avoiding overfitting.  This large number of hidden
units provides a variety of directions for improvement, and it has
been observed that training explores only a small subset of the
functions that could be represented with the parameters in each
layer (cf., \citep{denil2013predicting}).
Thus, we might gain useful insight into methods that optimize large
networks using a parametric gradient approach by viewing these methods
as an approximation to nonparametric, Fr\'echet gradient methods.



\acks{
Thanks to Dan Asimov for helpful discussions.
We gratefully acknowledge the support of the
NSF through grant IIS-1619362 and of the Australian Research Council
through an Australian Laureate Fellowship (FL110100281) and through
the Australian Research Council Centre of Excellence for Mathematical
and Statistical Frontiers (ACEMS).
}

\appendix
\section{Proof of Lemma~\protect\ref{lemma:smoothnessimplications}}
\label{a:smoothnessimplications}

Applying the gradient theorem for line integrals 
to each component of $h$ yields
    \[
      h(y) - h(x) = \int_0^1D h(x+t(y-x))(y-x)\,dt,
    \]
and this, together with smoothness, implies the first inequality:
      \begin{align*}
        \lefteqn{\left\|h(y) - \left(h(x)+D h(x)(y-x)\right)\right\|}
        & \\
          & = \left\| \int_0^1D h(x+t(y-x))(y-x)\,dt
            - D h(x)(y-x)\right\| \\
          & = \left\| \int_0^1\left(D h(x+t(y-x))-
            D h(x)\right)(y-x)\,dt \right\| \\
          & \le \int_0^1\left\| \left(D h(x+t(y-x))-
            D h(x)\right)(y-x)\right\|\,dt  \\
          & \le \alpha \int_0^1t\|y-x\|^2\,dt \\
          & = \frac{\alpha}{2} \|y-x\|^2.
      \end{align*}
  For the second, we write
      \begin{align*}
        h(y) - h(x) & = \int_0^1D h(x+t(y-x))(y-x)\,dt \\
            & = \int_0^1D h(0)(y-x)\,dt + 
            \int_0^1\left(Dh(0)-D h(x+t(y-x))\right)(y-x)\,dt, \\
      \end{align*}
  hence,
      \begin{align*}
        \left\|h(y) - h(x)\right\|
          & \le \left\| \int_0^1D h(0)(y-x)\,dt\right\|
            + \alpha\int_0^1\|x+t(y-x)\|\|y-x\|\,dt \\
          & = \|y-x\| + \alpha\|y-x\|R \\
          & = (1+\alpha R) \|y-x\|.
      \end{align*}



\bibliography{identity}

\end{document}